\documentclass[sigconf]{acmart}

\usepackage{mathrsfs}
\usepackage{amsmath}
\usepackage{amssymb}
\usepackage{amsthm}
\usepackage[ruled]{algorithm2e}
\usepackage{diagbox}
\usepackage{multirow}
\usepackage{tikz}

\newtheorem{lem}{Lemma}
\newtheorem{Thrm}{Theorem}

\AtBeginDocument{%
  \providecommand\BibTeX{{%
    \normalfont B\kern-0.5em{\scshape i\kern-0.25em b}\kern-0.8em\TeX}}}

\copyrightyear{2020}
\acmYear{2020}
\setcopyright{acmlicensed}\acmConference[CIKM '20]{Proceedings of the 29th ACM International Conference on Information and Knowledge Management}{October 19--23, 2020}{Virtual Event, Ireland}
\acmBooktitle{Proceedings of the 29th ACM International Conference on Information and Knowledge Management (CIKM '20), October 19--23, 2020, Virtual Event, Ireland}
\acmPrice{15.00}
\acmDOI{10.1145/3340531.3412156}
\acmISBN{978-1-4503-6859-9/20/10}

\begin{document}

\title{Denoising individual bias for a fairer binary submatrix detection}

\author{Changlin Wan}
\affiliation{\institution{Purdue University}}
\email{wan82@purdue.edu}

\author{Wennan Chang}
\affiliation{\institution{Purdue University}}
\email{chang534@purdue.edu}

\author{Tong Zhao}
\affiliation{\institution{Amazon}}
\email{zhaoton@amazon.com}

\author{Sha Cao}
\affiliation{\institution{Indiana University}}
\email{shacao@iu.edu}

\author{Chi Zhang}
\affiliation{\institution{Indiana University}}
\email{czhang87@iu.edu}

%\email{webmaster@marysville-ohio.com}
%\affiliation{%
%  \institution{Institute for Clarity in Documentation}
%  \streetaddress{P.O. Box 1212}

%\author{Lars Th{\o}rv{\"a}ld}
%\affiliation{%
%  \institution{The Th{\o}rv{\"a}ld Group}
%  \streetaddress{1 Th{\o}rv{\"a}ld Circle}
%\email{larst@affiliation.org}

\renewcommand{\shortauthors}{ et al.}

\begin{abstract}
Low rank representation of binary matrix is powerful in disentangling sparse individual-attribute associations, and has received wide applications. Existing binary matrix factorization (BMF) or co-clustering (CC) methods often assume i.i.d background noise. However, this assumption could be easily violated in real data, where heterogeneous row- or column-wise probability of binary entries results in disparate element-wise background distribution, and paralyzes the rationality of existing methods. We propose a binary data denoising framework, namely BIND, which optimizes the detection of true patterns by estimating the row- or column-wise mixture distribution of patterns and disparate background, and eliminating the binary attributes that are more likely from the background. BIND is supported by thoroughly derived mathematical property of the row- and column-wise mixture distributions. Our experiment on synthetic and real-world data demonstrated BIND effectively removes background noise and drastically increases the fairness and accuracy of state-of-the arts BMF and CC methods.
\end{abstract}




%%  \Description{Enjoying the baseball game from the third-base

\maketitle

\section{motivation}

\begin{figure}
    \centering
    \includegraphics[width=0.8\linewidth]{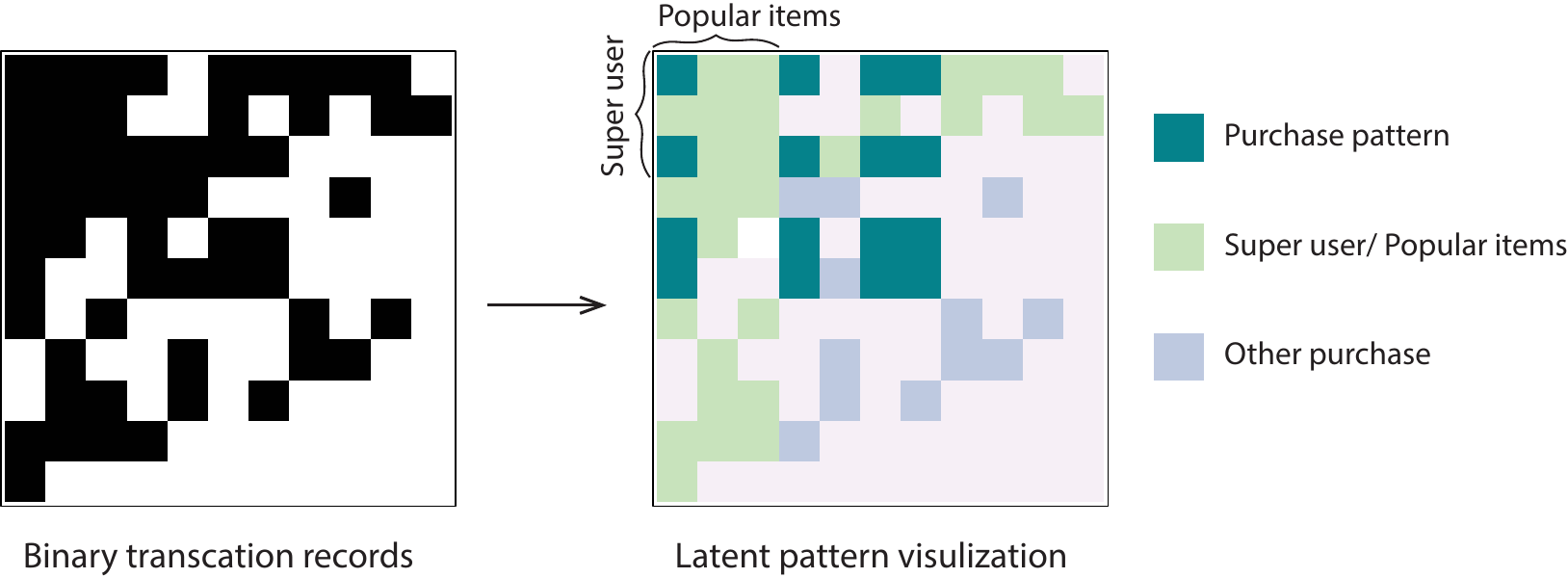}
    \caption{Individual bias in binary transaction records data}
    \label{fig:intro}
\end{figure}

Binary matrix has been commonly utilized in multiple fields. Low rank pattern in a binary matrix is defined as rank-1 sub matrices formed by the product of two binary bases. Comparing to continuous data, recent studies demonstrated the rank-1 sub-matrices in binarized data is more robust for mechanism interpretation or sub-space representation \cite{rukat2017bayesian,wan2019mebf}, because binary data in general bears reduced noise than continuous data. However, variations of the probability of 1s of rows or columns may lead to varied element-wise probability, causing a fairness issue in low rank representation of binary data  \cite{zhu2018fairness}.

An intuitive example is binary transaction records data (figure \ref{fig:intro}), in which 1s represent the purchase of items (each column) by users (each row). Different items or users are with varied activities in conducting purchasing. For example, super-users make more purchase, which can be independent to items, and popular items are more likely to be purchased. The transactions made between super users and popular items unnecessarily imply good recommendations since it can be simply caused by the high purchase chance. On the other hand, the group of items having a strong purchase preference within a certain group of users comparing to their background purchase rate is more valuable for recommendation. However, the fairness issue in the low rank representation of binary data due to varied element-wise background probability was rarely considered in existing formulations \cite{yao2017beyond}. 

Here, we propose BIND, a binary data denoising method via considering the data is generated from the mixture of to-be-identified rank-1 patterns and an unknown background of element-wise probability, plus i.i.d. errors. BIND estimates the mixture distribution of the probabilities of 1s from rank-1 patterns and background in each row and column, by which the rows or columns that are more likely with true rank-1 patterns are distinguished by the over-represented 1s comparing to the background.

Key contributions of this work include: (1) BIND is the first of this kind of binary data denoising method via considering non-identical background distribution, (2)  BIND can be easily implemented with state-of-the-arts BMF or CC methods for a fairer rank-1 pattern detection, and (3) rigorous mathematical derivations are provided to characterize the property of disparate background distribution.

\section{Background}
\subsection{Notations}
We denote matrix, vector and scalar by uppercase, bold lowercase and lowercase character \(X, \textbf{x},x\). Superscript with \(\times\) indicates dimensions, while subscript implies index, such as \(X^{m\times n}_{ij}\) and \(\textbf{x}^{m\times 1}_i\). \(P_{ij}\triangleq P(X_{ij}=1)\) denotes the element-wise probability of 1 at the element \(X_{ij}\). \(|\textbf{x}|\) and \(|X|\) represent the \(l1\) norm of vector and matrix, and \(\circ\) represents Hadamard product.

\subsection{Related work}
Existing methods of binary matrix low rank representation fall into two major categories, namely binary matrix decomposition (BMF) and co-clustering (CC). BMF aims to decompose a binary matrix as the product of two low rank binary matrix by maximizing its overall fitting to the original matrix. The formulation of BMF is thus generalized as \[X^{m\times n}=U^{m\times k}V^{k\times n}+E^{m\times n}\], where \(U\) and \(V\) are the low rank pattern matrices, and \(E\) is the flipping error with \(p(1\rightarrow 0)=p(0 \rightarrow 1)=p_0\). BMF problem is NP-hard, for which multiple heuristic algorithms have been developed. One representative method is ASSO, which retrieves candidate patterns by using row-/column-wise correlation \cite{miettinen2008discrete}. More recently, Bayesian probability measure and geometrical identification largely improved the efficiency and accuracy of BMF \cite{rukat2017bayesian,wan2019mebf}.

In contrast, the co-clustering (CC) method, also named as bi-clustering in statistics and computational biology, maximizes the enrichment of 1s in the detected patterns based on certain thresholds\cite{kaiser2008toolbox}. For given \(X^{m\times n}\), most CC methods aim to identify the cardinality of index set \(I_l\times J_l\), \(l=1,...,k\), where \(I_l \in \{1,...,m\}\) and \(J_l \in \{1,...,n\}\),\[s.t.\quad P_{ij}=\begin{cases} p_l,\, if \, i,j\in I_l\times J_l\\ p_0,\, if \, i,j\notin I_l\times J_l\end{cases} \forall l=1,...,k\]
Noted, both BMF and CC methods assume the binary data is formed by the sum of to-be-identified rank-1 submatrices and an i.i.d error, where individuals bias has not been investigated. 

\subsection{Problem formulation}
We consider the observed binary data with disparate element-wise background probability that is generated by:
\[X=U^{m\times k}V^{k\times n}+X^0+E'+E \tag{$\star$}\]
Compared with the formulation of BMF, \(X^0\) is the background matrix. \(E'\) is the pattern wise observation error that each element from pattern \(l\) has a probability of \(1-p_l\) to be zero, while the elements outside patterns will not be impacted, i.e., \(P^{E'}_{ij}(1\rightarrow 0)=1-p_l\), if \(i,j\in I_l\times J_l\), \(P^{E'}_{ij}(1\rightarrow 0)=0\), if \(i,j\notin I_l\times J_l\), \(\forall l = 1,...,k\). 

Under this definition, by considering \(X^0\) are 0, current BMF and CC described in 2.2 are special case of (\(\star\)), and were designed to handle the pattern observation error \(E'\) and elment-wise flipping error \(E\). Thus, the bottleneck of a fair binary submatrix detection lies in differentiating true patterns from the background \(X^0\). We consider the assumption of \(P(X_{ij}^0=1)\propto \textbf{p}^{0,r}_i\cdot \textbf{p}^{0,c}_j\) that can cover most of the binary data with disparate background, when \(X_{ij}^0\) are conditionally independent with fixed row or column index, like the purchase transaction data in figure \ref{fig:intro} with items of different popularity and users of different activity. We denote the row/column-wise background probability as \(\textbf{p}^{m\times 1,\,0, row}\) and \(\textbf{p}^{n\times 1,\,0, column}\), shorted as \(\textbf{p}^{0,r}\) and \(\textbf{p}^{0,c}\), where \(\textbf{p}^{0,r}_i\propto\hat{\textbf{p}}^{0,r}_i=\frac{|X^0_{i:}|}{n}\) and \(\textbf{p}^{0,c}_j\propto\hat{\textbf{p}}^{0,c}_j=\frac{|X^0_{:j}|}{m}\), and \(P(X_{ij}^0=1)\) can be unbiasedly estimated as \(\frac{|X^0_{i:}| \cdot |X^0_{:j}|}{ |X^0|}\).

\begin{figure}
    \centering
    \includegraphics[width=\linewidth]{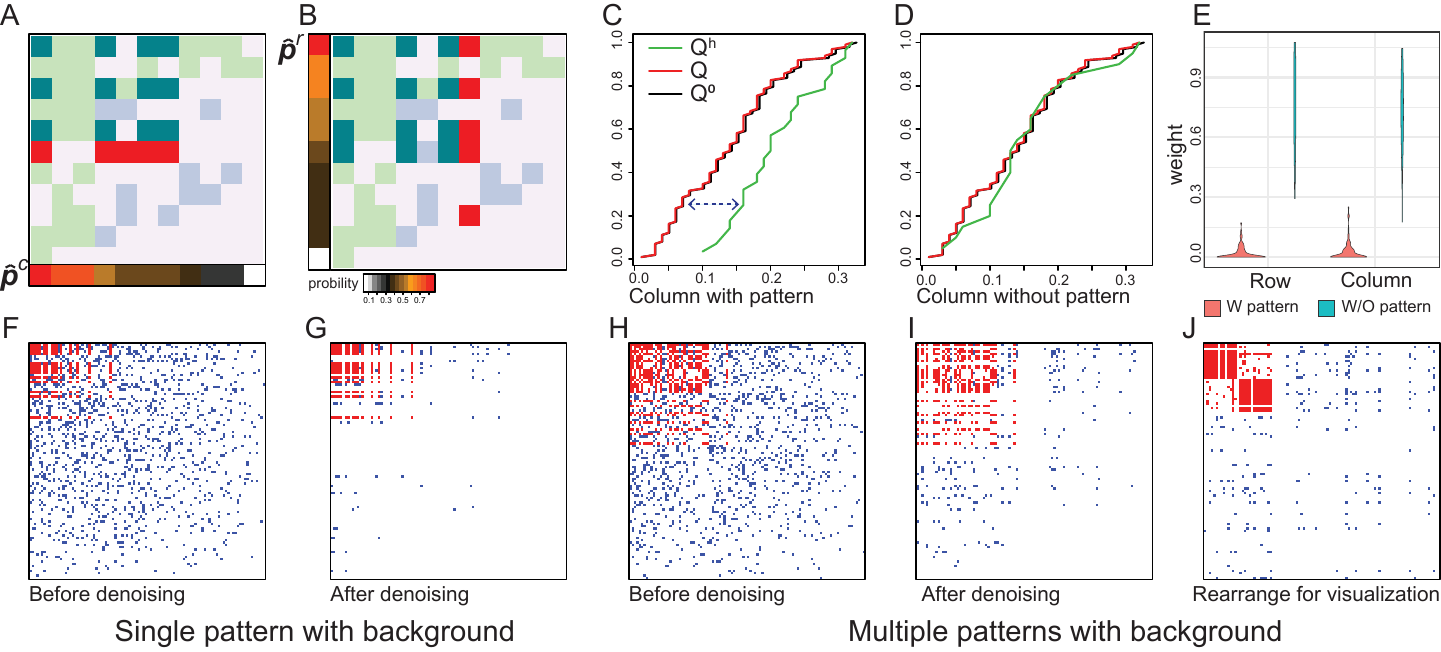}
    \caption{quantile shift denoising}
    \label{fig:algo}
\end{figure}

\section{BIND framework}
Here we propose the BIND\footnote{Code and material can be access at https://github.com/clwan/BIND} framework to identify the rank-1 patterns (\(U,V\)) from binary data \(X\) with disparate background \(X^0\). Denoting \(P(X_{ij}^0=1)\) as \(P_{ij}^0\), the element-wise probability \(P_{ij}\triangleq P(X_{ij}=1)\) can be derived as:
\[
P_{ij}=\begin{cases}P^0_{ij}\propto \textbf{p}^{0,r}_i\cdot\textbf{p}^{0,c}_j,\; ij\notin any\,I_l\times J_l\\1-(1-P_{ij}^0)(1-p_l)=p_{ij}^0+(1-p_{ij}^0)p_l,\; ij\in I_l\times J_l\end{cases} \tag{$\ast$} \]
Specifically, the row and column probability \(\textbf{p}^r_i\) and \(\textbf{p}^c_j\) can be estimated by \(\hat{\textbf{p}}^r_i=\frac{|X_{i:}|}{n}\) and \(\hat{\textbf{p}}^c_j=\frac{|X_{:j}|}{m}\). Noted, \(\textbf{p}^r\) and \(\textbf{p}^c\) are formed by the mixture distribution of \(\textbf{p}^{0,r}, \textbf{p}^{0,c}\) and \(p_l\). Analogous to BMF and CC problem, direct inference of \(\textbf{p}^{0,r}, \textbf{p}^{0,c}\) and \(p_l\) from \(\textbf{p}^r\) and \(\textbf{p}^c\) is NP-hard. As shown in Figure \ref{fig:algo}A-D, instead of computing \(\textbf{p}^{0,r}, \textbf{p}^{0,c}\) and \(p_l\), BIND identifies the rows and columns that are most likely conceiving patterns comparing to others. The elements of the intersection of the identified rows and columns more likely represent true rank-1 patterns (figure \ref{fig:algo}F-J). For this task, we introduce the quantile\_shift algorithm with thorough mathematical proof.

%\subsection{Quantile shift denoising}
\(Quantile\_shift\) algorithm is designed to distinguish rows or columns that are more likely conceiving rank-1 patterns. First, we introduce the concept of empirical distribution of row-/column-wise probability, denoted as \(\textbf{F}^r\) and \(\textbf{F}^c\) (figure \ref{fig:algo}A,B), which are sampled from \(\hat{\textbf{p}}^r\) and \(\hat{\textbf{p}}^c\) with probability \(P(\textbf{F}^{r}=\hat{\textbf{p}}^r_i)\propto\hat{\textbf{p}}^r_i\) and \(P(\textbf{F}^{c}=\hat{\textbf{p}}^c_j)\propto\hat{\textbf{p}}^c_j\). The observed probability of hits \(\textbf{F}^h\) of any row \(i_0\) or column \(j_0\) is defined by \(\textbf{F}^{h,r,i_0}=\{\hat{\textbf{p}}^c_j|j\ with\ X_{i_0j}=1\}\) and \(\textbf{F}^{h,c,j_0}=\{\hat{\textbf{p}}^r_i|i\ with\ X_{ij_0}=1\}\). Here \(\textbf{F}^r\) and \(\textbf{F}^c\) characterize the distribution of \(\hat{\textbf{p}}^r\) and \(\hat{\textbf{p}}^c\) of the 1s randomly drawn from \(\hat{\textbf{p}}^r\) and \(\hat{\textbf{p}}^c\). Intuitively, if a row or column conceives a distinct pattern, the quantile function \(Q^h\) of \(\textbf{F}^h\) will shift drastically from the quantile function \(Q^c\) of \(\textbf{F}^c\) or \(Q^r\) of \(\textbf{F}^r\) (figure \ref{fig:algo}C). On the other hand, \(Q^h\) will be similar to \(Q^c\) or \(Q^r\) if the row or column does not contain any pattern (figure \ref{fig:algo}D). Hence the shift between \(Q^h\) and \(Q^r\) or \(Q^c\) can serve as a weight \(s\) to differentiate the rows or columns more likely conceiving a pattern (figure \ref{fig:algo}E). Noted, here \(\textbf{F}^{r}\) and \(\textbf{F}^{c}\) serve as proxy of \(\textbf{F}^{0,r}\) and \(\textbf{F}^{0,c}\), which are the empirical distribution of the true background probability of \(\textbf{p}^{0,r}\) and \(\textbf{p}^{0,c}\). In the following content, we prove \(s\) approximates the pattern size within each row or column, i.e., \(s\approx |(UV+E')_{i:}|\ or\ |(UV+E')_{:j}|\) with certain bounds.

The input of \(Quantile\_shift\) algorithm include a row or column index \(i_0/j_0\), and \(\hat{\textbf{p}}^c\) or \(\hat{\textbf{p}}^r\), by which the empirical distribution \(\textbf{F}^c\) or \(\textbf{F}^r\) will be sampled, and the probability of hit of the row or column \(\textbf{F}^h\) will be computed. The output is weight \(s\) of the row or column. Without loss of generality, we illustrate the \(Quantile\_shift\) algorithm for computing the weight of row \(i_0\) below, and detailed mathematical proofs as follows:
\begin{algorithm}
\SetAlgoLined
\textbf{Inputs:}  Row index \(i_0\), Estimated column-wise probability \(\hat{\textbf{p}}^c\) \\
\textbf{Outputs:} Estimated weight of significance of row \(i_0\), \(s^r_{i_0}\) \\
\(Quantile\_shift (i_0,\hat{\textbf{p}}^c)\):\\
\(\textbf{F}^c\leftarrow sampled\,from\,\hat{\textbf{p}}^c\, with\, probability\, \hat{\textbf{p}}^c\) \\
\(\textbf{F}^{h}\leftarrow\{\hat{\textbf{p}}^c_j|j\ with\ X_{i_0j}=1\}\)\\
\(\textbf{F}^{(h)}\leftarrow sort(\textbf{F}^h)\), \(a\leftarrow length(\textbf{F}^h)\)\\
\(Q^c(p)=sup(b)\; s.t.\, \frac{|\textbf{F}^c<b|}{length(\textbf{F}^c)}\leq p\; and \; \frac{|\textbf{F}^c>b|+1}{length(\textbf{F}^c)}> p\)\\

\For{j=1...a}{
\uIf{\(\textbf{F}^{(h)}_j>Q^c(\frac{j}{a})\)}{
\(t_j\leftarrow the\ column\ index\ s.t.\ \textbf{F}^{(h)}_j=\hat{\textbf{p}}^c_{t_j}\ \& \  X_{i_0t_j}=1\)\\
\(s\leftarrow s+\frac{\textbf{F}^{(h)}_j-Q^c(\frac{j}{a})}{1-\hat{\textbf{p}}^c_{\textbf{t}_j}}\)
}
}
 \caption{Quantile\_shift}
\end{algorithm}

\begin{lem}
If \(\hat{\textbf{p}}^r\) and \(\hat{\textbf{p}}^c\)  are unbiased estimation of \(\textbf{p}^{0,r}\) and \(\textbf{p}^{0,c}\). The weight computed by quantile\_shift is an unbiased estimation of the sum of \(E(U^{m\times k}V^{k\times n}+E')\) with respect to that column or row. 
\end{lem}
\begin{proof}
If \(\hat{\textbf{p}}^r\) and \(\hat{\textbf{p}}^c\)  are unbiased estimation of \(\textbf{p}^{0,r}\) and \(\textbf{p}^{0,c}\), \(\textbf{F}^r\) or \(\textbf{F}^c\) generated from \(\hat{\textbf{p}}^r\) and \(\hat{\textbf{p}}^c\) form unbiased empirical distribution of row-/column-wise probability of 1s of \(X^0\), i.e. \(P(\textbf{F}^{0,r}=\textbf{p}^{0,r}_i)\propto\textbf{p}^{0,r}_i\) and \(P(\textbf{F}^{0,c}=\textbf{p}^{0,c}_j)\propto\textbf{p}^{0,c}_j\). Without loss of generality, we prove the lemma for the computation of the weight of the \(i_0\)th row. Denote \(\textbf{t}=\{j|X_{i_0j}=1\}\) and \(a=length(\textbf{t})\), by \(\textbf{Algorithm 1}\) and \((\ast)\), \(\forall\ j\in\{1,...,a\}:\)\\
If \(i_0t_j\notin any\, I_l\times J_l\),
\[
E(\textbf{F}^{(h)}_j-Q(\frac{j}{a}))=
E(\hat{\textbf{p}}^c_{\textbf{t}_j}-sup(b|\frac{|\textbf{F}^c<b|}{length(\textbf{F}^c)}\leq\frac{j}{a}))=0
\]
Else, \(i_0t_j \in I_l\times J_l\ for\ certain\ l\), 
\[
E(\textbf{F}^{(h)}_j-Q(\frac{j}{a}))=E(\hat{\textbf{p}}^c_{\textbf{t}_j}+(1-\hat{\textbf{p}}^c_{\textbf{t}_j})p_l-sup(b|\frac{|\textbf{F}^c<b|}{length(\textbf{F}^c)}\leq\frac{j}{a}))
\]
\[=(1-\hat{\textbf{p}}^c_{\textbf{t}_j})p_l\]
Such that
\[
E(\sum_{j=1}^a\frac{\textbf{F}^{(h)}_j-Q(\frac{j}{a})}{1-\hat{\textbf{p}}^c_{\textbf{t}_j}})=\sum_l\sum_{j=1}^a p_lI=|E(U^{m\times k}V^{k\times n}+E')_{i_0:}|\]
\end{proof}

\begin{lem}
For \(X\) in (\(\star\)), and \(P^0_{ij}\triangleq P(X^0_{ij})\propto \textbf{p}^{0,r}_i\cdot \textbf{p}^{0,c}_j\), the probability estimated by  \(\hat{p}^r_i=\frac{|X_{i:}|}{n}\) and \(\hat{p}^c_j=\frac{|X_{:j}|}{m}\) are bounded by \(|\hat{p}^r_i-\textbf{p}^{0,r}_i|\leq \frac{\sum_{l=1}^k\textbf{1}(i\in I_l)p_l|J_l|}{n}\), and \(|\hat{p}^c_j-\textbf{p}^{0,c}_j|\leq \frac{\sum_{l=1}^k\textbf{1}(j\in J_l)p_l|I_l|}{m}\).
\end{lem}
Lemma 2 can be derictly derived from \((\star)\) and \((\ast)\).
\begin{lem}
The weight of the \(i_0\)th row (or similarly \(j_0\)th column) is with a bias led by the biasedly estimated \(\hat{\textbf{p}}^c\) and \(\hat{\textbf{p}}^r\), which is bounded by \(E(s-|(UV+E')_{i_0:}|)\leq \frac{max(\textbf{F}^{c})+max(\frac{|E(UV+E'):j|}{m})(|\textbf{F}^{c}|+1)}{min(1-\textbf{p}^h)|\textbf{F}^{c}|}\).  
\end{lem}

We still use the compututaion of the \(i_0\)th row to illustrate the proof. The case for columns can be similarly derived.
\begin{proof}
By Lemma 2, \(\hat{\textbf{p}}^c\) is a biased estimation of \(\textbf{p}^{0,c}\), where \(\hat{\textbf{p}}^c_j=\frac{|X_{:j}|}{m}\geq \textbf{p}^{0,c}_j=\frac{|X^0_{:j}|}{m},\ j=1,...,m\). Hence \(\textbf{F}^{(h)}\geq \textbf{F}^{0,(h)}\), suggesting \(1-\textbf{F}^{0,(h)} \geq 1-\textbf{F}^{(h)}\) and \(Q^{c}(\frac{j}{a})\geq Q^{0,c}(\frac{j}{a})\), by which
\[
\left|\frac{\textbf{F}^{0,(h)}_j-Q^{0,c}(\frac{j}{a})}{1-\hat{\textbf{p}}^{0,c}_{\textbf{t}_j}}-\frac{\textbf{F}^{(h)}_j-Q^{c}(\frac{j}{a})}{1-\hat{\textbf{p}}^{c}_{\textbf{t}_j}}\right|\leq 2\left|\frac{\max\limits_{z\in(0,1)}\{Q^c(z)-Q^{0,c}(z)\}}{1-\hat{\textbf{p}}^{c}_{\textbf{t}_j}}\right|
\]

By lemma 2, the bias of \(|\hat{\textbf{p}}^c_j-\hat{\textbf{p}}^{c,0}_j|\) is bounded by \(\frac{|E(UV+E')_{:j}|}{m}\). So the max shift caused in the quantile function \(\max\limits_{z\in(0,1)}\{Q^c(z)-Q^{0,c}(z)\}\) is bounded by \(\frac{max(\hat{\textbf{p}}^c)+max(\frac{|E(UV+E')_{:j}|}{m})}{|\hat{\textbf{p}}^c|}+max(\frac{|E(UV+E')_{:j}|}{m})\). Hence the cumulative bias is bounded by
\[E(s-|E(UV+E')_{i_0:}|)\leq \frac{a(max(\hat{\textbf{p}}^c)+max(\frac{|E(UV+E')_{:j}|}{m})(|\hat{\textbf{p}}^c|+1)))}{min(1-\hat{\textbf{p}}^c)|\hat{\textbf{p}}^c|}\]
\end{proof}
Lemma 1 suggests |\(Q^h\)-\(Q^0\)| is an unbiased estimation of the expected number of 1s in the rank-1 patterns and Lemma 2-3 provide the bound of the bias of |\(Q^h\)-\(Q\)| when \(Q^0\) is biasedly estimated as \(Q\).

\begin{Thrm}[Quantile\_shift]
For a relative sparse binary matrix, the weight calculated by Quantile\_shift sufficiently characterizes the indices of the patterns with largest \(P_l|I_l|\) and \(P_l|J_l|\).
\end{Thrm}
\begin{proof}
For \(i0\)th row (or similarly for the \(j0\)th column), 
\[E(s-|(UV+E')_{i_0:}|)\leq \frac{a(max(\hat{\textbf{p}}^c)+max(\frac{|E(UV+E')_{:j}|}{m})(|\hat{\textbf{p}}^c|+1)))}{min(1-\hat{\textbf{p}}^c)|\hat{\textbf{p}}^c|}\]\[\approx \frac{a}{min(1-\hat{\textbf{p}}^c)}max\{\frac{max(\hat{\textbf{p}}^c)}{|\hat{\textbf{p}}^c|}, max(\frac{|E(UV+E')_{:j}|}{m})\}\]
, suggests that when the input matrix and rank-1 patterns are relatively sparse, the weight \(s\) approximates \((UV+E)_{i_0:}\), i.e. largest values in \(\textbf{s}^r\) and \(\textbf{s}^c\) correspond to the rows and  columns of the patterns with largest \(P_l|I_l|\) and \(P_l|J_l|\). 
\end{proof}

BIND framework is developed to implement \(Quantile\_shift\) algorithm with a BMF or CC method, denoted as \(\mathcal{F}\), for a fairer rank-1 pattern identification under the formulation of (\(\star\)). As illustrated in figure \ref{fig:algo}F-J, \(Quantile\_shift\) denoises the majority of the background signal and enables a BMF or CC method better detects \(U^{m\times k}\) and \(V^{k\times n}\). A cutoff \(\tau\) is needed to differentiated the weight of the rows or columns with true patterns (figure \ref{fig:algo}E). Empirically, \(\tau\) could be set from 0.05 to 0.1 in BIND algorithm. 

BIND is capable for one direction denoising. The \(Quantile\_shift\) algorithm is \(O(n)\) or \(O(m)\) for row or column weight computation and the BIND algorithm is \(O(mn)\), which is smaller than most of current BMF and CC methods. The BIND algorithm is detailed below:   

\begin{algorithm}
\SetAlgoLined
\textbf{Inputs:}  Input data \(X^{m\times n}\), Threshold \(\tau\), BMF/CC method \(\mathcal{F}\)\\
\textbf{Outputs:} Pattern matrices \(U^{m\times k}\) and \(V^{k\times n}\) \\
\(BIND (X, \tau,\mathcal{F})\):\\
\(X_{use}\leftarrow 0\cdot X\), \(\textbf{s}^r\leftarrow \textbf{0}^{m\times 1}\), \(\textbf{s}^c\leftarrow \textbf{0}^{n\times 1}\)\\
\(\hat{\textbf{p}}^r_i=\frac{|X_{i:}|}{n}\, \forall i=1,...,m\) and \(\hat{\textbf{p}}^c_j=\frac{|X_{:j}|}{m}\, \forall j=1,...,n\)\\
\For{i=1...m}{
\(\textbf{s}^r_i=Quantile\_shift(i,\hat{\textbf{p}}^c)\)
}
\For{j=1...n}{
\(\textbf{s}^c_j=Quantile\_shift(j,\hat{\textbf{p}}^r)\)
}
\(I^r\leftarrow I({\textbf{s}^r>\tau})\), \(I^c\leftarrow I({\textbf{s}^c>\tau})\), \(X_{use}\leftarrow X \circ (I^r\cdot{I^c}^T)\)\\
\(U,V\leftarrow \mathcal{F}(X_{use},...)\)

 \caption{BIND}
\end{algorithm}

\section{Experiment}
In this section, we evaluate the performance of BIND on synthetic and real-world data sets across different data scenarios. We demonstrate the implementation of BIND with different BIND BMF and CC methods can significantly improve their fairness in detecting rank-1 pattern from binary matrix with disparate background probability. We also highlight the application of BIND framework for better result interpretation on real-world Movielens data.

We simulate synthetic data sets \(X^{100\times 100}\) with fixed size by following (\(\star\)): \(X=U^{m\times k}V^{k\times n}+E'+X^0+E\), with different pattern size \(\in\{10,15,20\}\), pattern number \(k\in\{1,2\}\), observation error \(p_k\in\{0.8,0.9,1.0\}\), background probability \(\textbf{p}^{0,r},\textbf{p}^{0,c}\), and element-wise flipping error \(p_0 \in\{0,0.05\}\). Specifically, background probabilities were generated from uniform distribution \(\textbf{p}^{0,r},\textbf{p}^{0,c}\sim U[0.1,p]\), where \(p\in\{0.5,0.6,0.7\}\) corresponds to different background probabilities. Altogether, we deem 108 data scenarios from the above parameter settings and simulated 30 replicates for each scenario to form a test-bed. Jaccard index \(\textbf{D}=\frac{|X\cap UV|}{|X\cup UV|}\ (X=\ original\ or\ denoised\ data)\) is used as the evaluation metric. For each data scenario, denoising performance is evaluated by the averaged Jaccard index on the 30 replicates. We first compare the performance with respect to different significance threshold \(\tau=\{0,0.05-1\}\), where \(\tau=0\) represents the data without denoising. As shown in figure \ref{fig:simu}A, the denoising process on average increased the Jaccard index by 2.6 fold and denoising efficiency is slightly increased with \(\tau\). Table 1 lists the denoising performance with respect to different number of patterns \(k\), background probability\(p\) and observation probability \(p_k\), where pattern size is set as 15 and \(\tau=0.1\).

We benchmark BIND by implementing with recently developed BMF method LOM and CC method Biclust, which showed top performance among similar state-of-the-arts methods \cite{rukat2017bayesian,kaiser2008toolbox}. The implementation of BIND largely increased the accuracy in detecting true patterns, which results in an averaged 7.5 (LOM) and 2.6 (Biclust) fold increase of the Jaccard index  (figure \ref{fig:simu}B,C) .

We also demostrate that BIND increases the interpretation and denoising in real-world Movielens data, in which \(X_{ij}=1\) represents the interest of user \(i\) (row) in rating/watching movie \(j\) (column). Category label of each movie is provided. Intuitively, disparate background probablities naturally exist in this data due to different popularity of movies and activity of users. Data is divided into four regions by the \(I^r\) and \(I^c\) computed in \(\textbf{Algorithm 2}\) (figure \ref{fig:simu}D,E), where \(\textcircled{1}\) is the region most likely with patterns, and \(\textcircled{2}\), \(\textcircled{3}\) and \(\textcircled{4}\) are denoised regions. Users in region \(\textcircled{1}\) watched more movies but less categories comparing to other regions (figure \ref{fig:simu}F), suggesting potential recommendation. In addition, region \(\textcircled{1}\) has smallest dispersion of the number of rated movies with respect to different categories, suggesting more stable rating preference of users towards their preferred movie types in this region (figure \ref{fig:simu}G).  

\begin{figure}
    \centering
    \includegraphics[width=\linewidth]{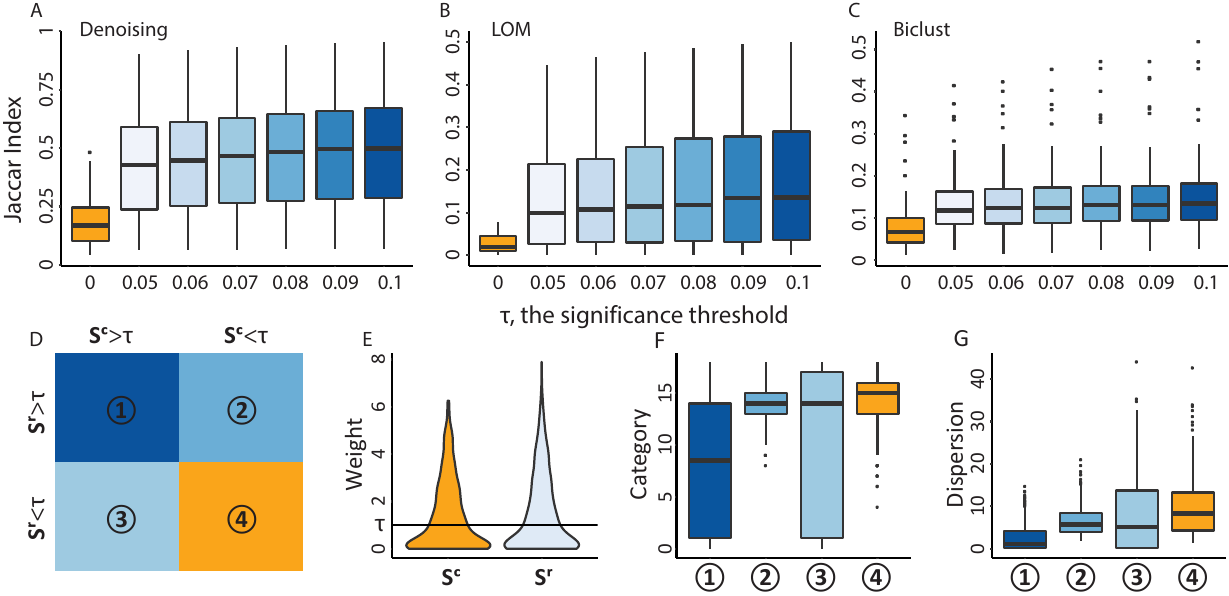}
    \caption{Performance on simulated and Movielens data}
    \label{fig:simu}
\end{figure}

\begin{table}
  \begin{tabular}{m{0.12\linewidth}|m{0.1\linewidth}m{0.1\linewidth}m{0.1\linewidth}m{0.1\linewidth}m{0.1\linewidth}m{0.1\linewidth}}
    \toprule
    \multirow{2}{*}{\diagbox{\(p\)}{\(p_k\)}}&\multicolumn{3}{c}{single pattern}&\multicolumn{3}{c}{Multiple pattern}\\
    &0.8&0.9&1.0&0.8&0.9&1.0\\
    \midrule
    0.5& 0.17/0.67 & 0.18/0.79 & 0.20/0.88 & 0.28/0.59 & 0.31/0.73 & 0.34/0.84 \\
    0.6&0.13/0.48&0.14/0.61&0.16/0.73&0.23/0.47&0.26/0.59&0.28/0.69\\
    0.7&0.11/0.29&0.11/0.37&0.13/0.47&0.19/0.34&0.21/0.40&0.22/0.52\\
  \bottomrule
\end{tabular}
  \caption{Jaccard index before/after denoising}
  \label{tab:example}
\end{table}

\section{acknowledgments}
This work was supported by R01 award \#1R01GM131399- 01, NSF IIS (N0.1850360), Showalter Young Investigator Award from Indiana CTSI and Indiana University Grand Challenge Precision Health Initiative.
\bibliographystyle{ACM-Reference-Format}
\bibliography{BIND}

\end{document}